\pdfoutput=1
\documentclass{article}

\usepackage[affil-it]{authblk}
\usepackage[utf8]{inputenc} 
\usepackage[T1]{fontenc}    

\usepackage[bitstream-charter, cal=cmcal]{mathdesign}
\usepackage{amsmath}
\usepackage[scaled=0.92]{PTSans}

\usepackage[
  paper  = letterpaper,
  left   = 1.35in,
  right  = 1.35in,
  top    = 1.0in,
  bottom = 1.0in,
  ]{geometry}

\usepackage[usenames,dvipsnames,table]{xcolor}
\definecolor{shadecolor}{gray}{0.9}

\usepackage[final,expansion=alltext]{microtype}
\usepackage[english]{babel}
\usepackage[parfill]{parskip}
\usepackage{afterpage}
\usepackage{framed}

{\endMakeFramed}

\usepackage{lineno}

\usepackage{ragged2e}

\newcounter{parcount}

\usepackage{graphicx}
\usepackage{wrapfig}
\usepackage[labelfont=bf,font=footnotesize,width=.9\textwidth]{caption}
\usepackage[format=hang]{subcaption}

\usepackage{enumitem} 

\usepackage{booktabs,multirow,multicol}       
\usepackage[separate-uncertainty=true,multi-part-units=single]{siunitx} 

\usepackage[algoruled]{algorithm2e}
\usepackage{listings}
\usepackage{fancyvrb}
\fvset{fontsize=\normalsize}

\usepackage{natbib}

\usepackage[colorlinks,linktoc=all]{hyperref}
\usepackage[all]{hypcap}
\hypersetup{citecolor=MidnightBlue}
\hypersetup{linkcolor=MidnightBlue}
\hypersetup{urlcolor=MidnightBlue}

\usepackage[nameinlink,capitalise]{cleveref}

\usepackage{bibunits}
\defaultbibliographystyle{abbrvnat} 
\defaultbibliography{refs}

\usepackage{url}

\usepackage[acronym, nowarn]{glossaries}

\lstdefinestyle{mystyle}{
    commentstyle=\color{OliveGreen},
    keywordstyle=\color{BurntOrange},
    numberstyle=\tiny\color{black!60},
    stringstyle=\color{MidnightBlue},
    basicstyle=\ttfamily,
    breakatwhitespace=false,
    breaklines=true,
    captionpos=b,
    keepspaces=true,
    numbers=left,
    numbersep=5pt,
    showspaces=false,
    showstringspaces=false,
    showtabs=false,
    tabsize=2
}
\usepackage{microtype} 

\usepackage{url}
\usepackage{multirow}
\usepackage{booktabs}
\usepackage{amsfonts}
\usepackage{nicefrac}
\usepackage{microtype}
\usepackage{float}
\usepackage{makecell}

\usepackage{centernot}
\usepackage{amsthm}
\usepackage{nicefrac}       
\usepackage{mathtools}
\usepackage{amsbsy}
\usepackage{amstext}
\usepackage{amsthm}
\usepackage{thmtools}
\usepackage{thm-restate}

\usepackage{bm}

\usepackage{dsfont}

\begingroup
    \makeatletter
    \@for\theoremstyle:=definition,remark,plain\do{%
        \expandafter\g@addto@macro\csname th@\theoremstyle\endcsname{%
            \addtolength\thm@preskip\parskip
            }%
        }
\endgroup

\crefname{lemma}{lemma}{lemmas}
\Crefname{lemma}{Lemma}{Lemmas}
\crefname{thm}{theorem}{theorems}
\Crefname{thm}{Theorem}{Theorems}
\crefname{prop}{proposition}{propositions}
\Crefname{prop}{Proposition}{Propositions}
\crefname{assumption}{assumption}{assumptions}
\crefname{assumption}{Assumption}{Assumptions}

\renewcommand{\mid}{~\vert~}

\usepackage{booktabs,arydshln}
\makeatletter
\def\adl@drawiv#1#2#3{%
        \hskip.5\tabcolsep
        \xleaders#3{#2.5\@tempdimb #1{1}#2.5\@tempdimb}%
                #2\z@ plus1fil minus1fil\relax
        \hskip.5\tabcolsep}
\newcommand{\cdashlinelr}[1]{%
  \noalign{\vskip\aboverulesep
           \global\let\@dashdrawstore\adl@draw
           \global\let\adl@draw\adl@drawiv}
  \cdashline{#1}
  \noalign{\global\let\adl@draw\@dashdrawstore
           \vskip\belowrulesep}}
\makeatother

\renewcommand{\epsilon}{\varepsilon}

\declaretheorem[style=plain,numberwithin=section,name=Theorem]{theorem}

\declaretheorem[style=plain,sibling=theorem,name=Proposition]{proposition}

\declaretheorem[style=definition,name=Assumption]{assumption}
\declaretheorem[style=definition,sibling=theorem,name=Example]{example}

\newenvironment{example*}
 {\pushQED{\qed}\example}
 {\popQED\endexample}
\numberwithin{equation}{section}

\DeclareMathOperator*{\argmin}{arg\,min}
\definecolor{WowColor}{rgb}{.75,0,.75}
\definecolor{SubtleColor}{rgb}{0,0,.50}

\newcounter{margincounter}

\DeclarePairedDelimiterX\Set[1]{\lbrace}{\rbrace}%
{  #1 }

\newcommand{\papertitle}{Supercharging Bayesian Inference with Reliable AI-Informed Priors}

\title{\papertitle}

\author[1]{Jongwoo Choi}
\author[2]{Sean O'Hagan}
\affil[1]{Department of Statistics, University of Connecticut}
\affil[2]{Department of Statistics, University of Chicago}

\begin{document}

\maketitle

\begin{bibunit}
\begin{abstract}
Modern predictive systems encode beliefs that can act as useful prior information for statistical inference in data-limited settings. Using them for prior construction introduces a tradeoff: an informative prior built from a predictive model can sharpen inference from limited data, but also risks propagating error from the model into the posterior. We propose a framework for AI-informed prior elicitation that mitigates this tension by rectifying the AI-induced law that generates synthetic data before using it to inform a prior. The rectified law can be embedded into synthetic data-driven prior elicitation techniques, including as a base measure in a Dirichlet process (DP) prior on the data-generating process. We refer to the resulting prior and corresponding posterior as the \emph{rectified AI prior} and \emph{rectified AI posterior}. We establish Gaussian asymptotics for the rectified AI posterior under non-vanishing prior strength and derive a first-order expression for its centering bias. Our rectified AI priors substantially reduce bias compared to standard approaches, improve the coverage of credible intervals, and make AI-powered prior information more reliable. We additionally apply the rectified AI prior to a real skin disease classification task and show that it can meaningfully boost predictive performance.
\end{abstract}

\section{Introduction}

Machine learning systems are often capable of producing outputs that are informative about quantities of scientific interest, as evidenced by recent advancements in protein structure prediction~\citep{jumper2021highly}, weather forecasting~\citep{lam2023learning}, and genomics~\citep{vaishnav2022evolution}. Utilizing these outputs in data-limited settings can allow for substantial improvements in statistical efficiency that may lead to discoveries. This is especially relevant in applications where outcome variables are defined via human annotation, where gold-standard labels are scarce. In many such cases, off-the-shelf predictive systems can impute these labels with impressive accuracy \citep{ziems2024can}.

The Bayesian perspective intends to harness the predictive power contained within machine learning systems by leveraging them for prior elicitation. This approach is philosophically natural---the information that we extract from such systems is not itself a replacement for data, but rather external information about plausible structure in the underlying data-generating process~\citep{dale2026synthetic}. Indeed, we view machine learning systems as encoding beliefs that can guide inference, while observed data maintain the responsibility of updating and correcting those beliefs. 

However, this perspective introduces a fundamental tension. A prior that has been heavily informed by the outputs of a machine learning system has the potential to improve statistical inference by concentrating posterior mass in plausible regions of the parameter space, reducing posterior variance and yielding sharper inferences from limited data. On the other hand, such a prior is vulnerable to systematic error if the beliefs of the ML system are miscalibrated. The crux of the matter is then how to incorporate AI systems in prior construction in a manner that can both preserve the variance reduction boon of an informative prior without sacrificing inferential reliability.

In this article, we develop methodology for AI-informed prior elicitation that addresses this tension. We consider the following desiderata. Primarily, the framework for prior construction should be simple, flexible, and easily allow the practitioner to incorporate a variety of ML systems. Posterior sampling should be computationally convenient and not reliant on restrictive assumptions about the predictive system. In addition, the informed priors that are produced should improve statistical efficiency, but in such a way that controls the bias introduced by the ML system and yields reliable downstream inference and predictions. 
We address this by \emph{rectifying} the AI-induced law that generates the synthetic data itself, making it more reliable for calibrating a prior in downstream inference and prediction tasks. This rectified law itself can then be incorporated in data-driven prior constructions, including the power prior~\citep{chen2000power}, expected-posterior prior~\citep{perez2002expected}, catalytic prior~\citep{huang2025catalytic}, and AI prior~\citep{o2025ai}. While our rectified synthetic data-generating process can be incorporated into any prior leveraging synthetic data, we prioritize using it as the base measure of a Dirichlet process (DP) prior on the data-generating process---termed the \emph{AI prior}---due to its versatility and computational convenience for posterior sampling.

We show that rectifying the AI-based synthetic data-generating law allows us to construct informed priors that retain inferential reliability. In addition, these priors can meaningfully boost performance in downstream predictive tasks. Our approach allows us to avoid computationally intensive calibration procedures for the prior influence that trade off variance reduction with bias reduction, and allows us to instead achieve variance reduction without subjecting ourselves to the same level of systematic bias. \cref{fig:demo} highlights the effect of our rectified AI prior when compared to the raw AI prior when estimating quantiles of gene expression data. 

Our contributions are as follows: 
In~\cref{sec:rectified-base-measure}, we introduce our procedure for constructing rectified AI-informed priors. 
We establish asymptotics showing when the resulting prior remains reliable under nonvanishing prior strength. In~\cref{sec:rectifier-taxonomy}, we provide practical guidance on how to rectify the AI base measure in certain classes of problems. 
We then provide empirical verification of our methodology in~\cref{sec:experiments}, where we apply our rectified AI priors to infer quantiles of gene expression levels and of a regression coefficient on census data. In addition, we show that our rectified AI prior can substantially improve predictive power in a skin-disease classification problem. 

\begin{figure}
    \centering
    \includegraphics[width=\linewidth]{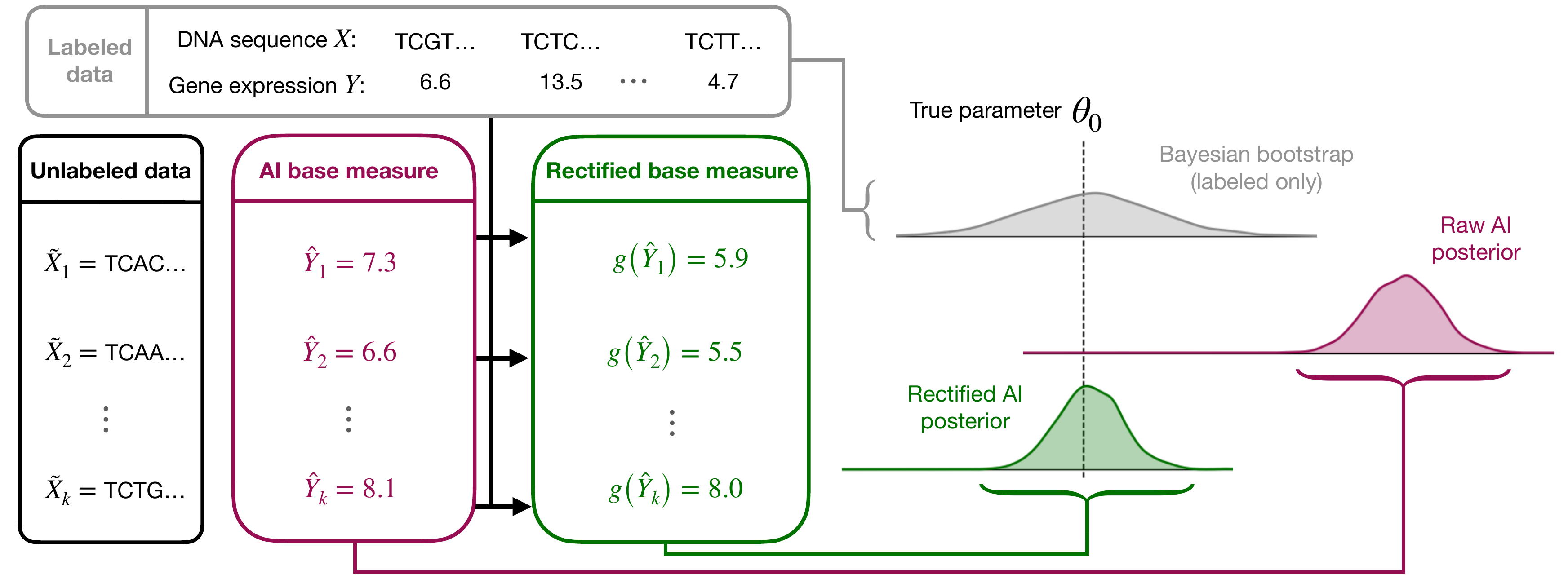}
    \caption{Rectifying the base measure yields a posterior that provides more accurate inference (removing bias), while retaining statistical efficiency advantages when compared to objective Bayes procedures like the Bayesian bootstrap.}
    \label{fig:demo}
\end{figure}

\section{Preliminaries: AI-Powered Bayesian Inference}\label{sec:prelim}

We briefly review the AI-powered Bayesian inference framework. Consider the data-generating process $\mathcal{D}_n:=(X_1,Y_1),\ldots,(X_n,Y_n)\overset{\mathrm{iid}}{\sim}P_0$, where $X_1\in\mathcal{X}$, $Y_1\in\mathcal{Y}$, and $P_0$ denotes the true data-generating distribution on $\mathcal{X}\times\mathcal{Y}$. For any probability measure $P$ and measurable function $h$, write $P h = \int h(x,y) \, \mathrm{d}P(x,y)$. We are interested in Bayesian statistical inference on a functional $\theta(P_0)$ based on the independent sample. We have access to an auxiliary probability distribution $P_{\mathrm{AI}}$ capable of generating samples in $\mathcal{X}\times\mathcal{Y}$. In practice, we hope that the samples arriving from $P_{\mathrm{AI}}$ resemble those from $P_0$ so that they may aid in inference for $\theta(P_0)$. 

The law $P_{\mathrm{AI}}$ is often induced through a machine learning (or generative AI) system. In many cases, it admits a factorization $P_{\mathrm{AI}}(\mathrm{d}x,\mathrm{d}y)=P_{\mathrm{AI}}^X(\mathrm{d}x) P_{\mathrm{AI}}^Y(\mathrm{d}y\mid x)$. A common special case takes $P_{\mathrm{AI}}^X$ to be the empirical distribution of an available sample of unlabeled covariates, while the conditional law $P_{\mathrm{AI}}^Y(\cdot\mid x)$ is induced by a predictive model. When $P_{\mathrm{AI}}^Y(\cdot\mid x)$ is degenerate, this recovers the setting of \emph{inference with predicted data} (equivalently, that of \emph{prediction-powered inference}~\citep{angelopoulos2023prediction}).

AI-powered Bayesian inference~\citep{o2025ai} leverages $P_{\mathrm{AI}}$ to elicit an informative prior on the data-generating process $P_0$, which will in turn aid in inference on $\theta(P_0)$. In particular, we have
\begin{align*}
    P_0\sim\mathrm{DP}(\alpha,P_{\mathrm{AI}})\\
    (X_1,Y_1),\ldots,(X_n,Y_n)&\overset{\mathrm{iid}}{\sim}P_0 
\end{align*}
where $\mathrm{DP}(a,G)$ denotes the Dirichlet process with concentration parameter $a$ and base measure $G$. When $\theta(P_0)$ is a population risk minimizer, i.e. $\theta(P_0) \in \argmin_{\theta\in\Theta}P_0\ell_\theta$ where $\ell_\theta:\mathcal{X}\times\mathcal{Y}\to\mathbb{R}$ is a loss function, then the posterior bootstrap~\citep{fong2019scalable} allows for embarrassingly parallel sampling from the posterior $\theta(P_0)\mid\mathcal{D}_n$. If $P_{\mathrm{AI}}$ is an empirical distribution on $(\tilde{X}_1,\tilde{Y}_1),\ldots,(\tilde{X}_k,\tilde{Y}_k)$, then sampling is exact by leveraging conjugacy of the DP prior via 
\[
\argmin_{\theta\in\Theta}\left\{\sum_{i=1}^n w_i \ell_\theta(X_i,Y_i)+\sum_{j=1}^k \tilde{w}_j \ell_\theta\big(\tilde{X}_j,\tilde{Y}_j\big)\right\}\sim \theta(P_0)\mid \mathcal{D}_n
\]
where $(w_1,\ldots,w_n,\tilde{w}_1,\ldots,\tilde{w}_k)\sim\mathrm{Dirichlet}(1,\ldots,1,\alpha/k,\ldots,\alpha/k)$. When $P_{\mathrm{AI}}$ is not a finite mixture of atoms, approximate sampling is analogously available by using a finite truncation.

The concentration parameter $\alpha$ represents prior strength and can be calibrated for inferential reliability based on an adapted version of the generalized posterior calibration algorithm~\citep{syring2019calibration}. This amounts to scaling how much influence to assign the AI base measure, and can result in substantially underusing an otherwise informative AI base measure if it is subject to systematic but correctable error.

\section{The Rectified AI Base Measure}\label{sec:rectified-base-measure}
A challenge in AI-informed prior elicitation is not just to borrow information from an AI-induced law, but to ensure that the law being borrowed from is reliable for the inferential target. When the AI-induced law is systematically biased, a prior with non-vanishing strength can shift the posterior distribution of the induced functional $\theta(P)$ away from the target parameter $\theta_0 = \theta(P_0)$. To mitigate this bias, we propose using labeled calibration data to construct a rectified base measure. Centering a DP prior at this rectified base measure rather than at the raw AI base measure can reduce the induced bias while retaining the efficiency gains from AI-powered prior information.

The manner by which we construct a rectified base measure may depend on the dimensionality and structure of the data, as well as on whether the parameter of interest is a functional of the marginal distribution of $Y$ or depends on the conditional generating process of $Y$ given $X$. We describe a general prescription below, and defer our practical recommendations to \cref{sec:rectifier-taxonomy}.

A \emph{rectifier} is a data-dependent map from the set of probability measures on $\mathcal{X}\times\mathcal{Y}$ to itself. We consider a set of candidate rectifiers $\{T_\eta : \eta \in \mathcal H\}$. For a given $\eta \in \mathcal H$, the corresponding rectified base measure is $P_\eta = T_\eta (P_{\mathrm{AI}})$. Typically, we select a rectifier using a labeled calibration sample of size $m$, i.e.
    $\hat{\eta}_m\in\arg\min_{\eta\in\mathcal{H}} m^{-1}\sum_{i=1}^m L_{\eta}(X_i,Y_i)$
where $L_{\eta}$ is some fitting criterion that rewards values of $\eta$ that make $P_\eta$ better resemble $P_0$ for estimating $\theta$.
We write $P_{\hat\eta_m} = T_{\hat \eta_m} (P_{\mathrm{AI}})$ for the fitted rectified base measure. 

In particular, suppose the parameter of interest $\theta_0$ is identified by an estimating equation $P_0 g_{\theta_0}=0$, where $g_\theta:\mathcal{X}\times\mathcal{Y}\to\mathbb{R}^d$ could be a gradient or subgradient of a convex loss $\ell_\theta$. A fitted rectified base measure $P_{\hat\eta_m}$ induces a corresponding estimating equation $P_{\hat{\eta}_m}g_\theta=0$.
The natural goal of rectification for the parameter $\theta_0$ is therefore to make the estimating equation induced by $P_{\hat{\eta}_m}$ agree with that under $P_0$.
Accordingly, our goal is to choose a candidate rectifier such that $\|(P_0-P_{\hat{\eta}_m})g_\theta\|$ is small, in particular for $\theta$ near $\theta_0$. 
\cref{sec:theory} formalizes this idea by showing that the posterior bootstrap limit is preserved after rectification, while the leading centering bias is controlled by the score discrepancy $(P_0-P_{\hat{\eta}_m})g_\theta$.

\subsection{Theory}\label{sec:theory}
The target parameter is defined as a minimizer of a population risk $\theta_0 \in \argmin_{\theta\in\Theta} P_0\ell_\theta$, where $\Theta \subset \mathbb{R}^d$ is the parameter space and $\ell_\theta = \ell(\theta, X, Y)$. For a fixed $\gamma >0$, define the rectified oracle target and the rectified empirical risk minimizer respectively as 
\begin{equation*}
    \theta_{0,m}^{\gamma} \in \argmin_{\theta\in\Theta}\left\{P_0\ell_\theta+\gamma P_{\hat{\eta}_m}\ell_\theta\right\}, \qquad \hat{\theta}_{n,m}^\gamma \in \argmin_{\theta\in\Theta}\left\{P_n\ell_\theta + \gamma P_{\hat{\eta}_m}\ell_\theta\right\}.
\end{equation*}
Here, $P_n$ represents the empirical distribution of the data sample of size $n$: $P_n h = \tfrac{1}{n}\sum_{i=1}^n h(X_i, Y_i)$.
To avoid technical complications from reusing the same labels twice, the theory is stated under sample splitting: the observations used to construct $P_{\hat{\eta}_m}$ are independent of those used to form $P_n$. 

The scalar $\gamma = \alpha / n$ controls the strength of the rectified AI prior relative to the labeled inference sample, where $\alpha$ is the DP concentration parameter. Setting $\gamma=1$ gives the $P_{\hat\eta_m}$ the same total weight as the $n$ labeled observations used to form $P_n$, while larger values of $\gamma$ correspond to stronger borrowing from $P_{\hat \eta_m}$. Let $\Pi_{\mathrm{PB}}^{\hat{\eta}_m}$ denote the posterior bootstrap law induced by the prior $DP(\gamma n, P_{\hat{\eta}_m})$. For the differentiable loss, write $g_\theta = \nabla_\theta \ell_\theta$ and $\dot g_\theta = \nabla^2_\theta \ell_\theta$. 

\cref{thm1} is a conditional analogue of the posterior bootstrap theorem (Theorem 1; \citealp{o2025ai}) with $P_{\hat\eta_m}$ replacing the raw AI base measure. Sufficient conditions for consistency of the corresponding $\theta^\gamma_{0,m}$ and $\hat\theta^\gamma_{n,m}$ are given in \cref{prop1}. We state a convenient set of sufficient regularity conditions in \cref{supp:cond2}. All proofs are deferred to \cref{supp:proofs}.

\begin{theorem}\label{thm1}
    Suppose \cref{supp:cond2} holds. Define $J_1(\theta) = P_0\dot g_\theta$, $J_{2,m}(\theta) = P_{\hat\eta_m}\dot g_\theta$, and $I_1(\theta) = P_0(g_\theta g_\theta^\mathtt{T})$, $I_{2,m}(\theta) = P_{\hat\eta_m}(g_\theta g_\theta^\mathtt{T})$. Let 
    \begin{equation*}
        J_m := \frac{1}{1 + \gamma} \big\{ J_1(\theta_{0,m}^\gamma) + \gamma J_{2,m}(\theta_{0,m}^\gamma) \big\}, \quad I_m := \frac{1}{1 + \gamma} \big\{ I_1(\theta_{0,m}^\gamma) + \gamma I_{2,m}(\theta_{0,m}^\gamma) \big\}.
    \end{equation*}
    If $Z_m \mid P_{\hat \eta_m} \sim \mathcal N(0, J_m^{-1}I_m J_m^{-1})$, then, conditionally on the calibration sample, for every Borel set $A \subset \mathbb R^d$ such that $P(Z_m \in \partial A \mid P_{\hat \eta_m}) =0$, we have
    \begin{equation*}
        \Pi_{\mathrm{PB}}^{\hat{\eta}_m} \Big[ \sqrt{n(1+\gamma)} \big( \theta - \hat{\theta}_{n,m}^\gamma \big) \in A \Big] \xrightarrow[n,m \to \infty]{\mathrm{p}} P(Z_m \in A \mid P_{\hat \eta_m}).
    \end{equation*}

    When $P_{\hat\eta_m}$ is represented by a finite artificial sample size $M$, the same conclusion holds with the finite-base analogues, under the regime for $M$ used in \citet{o2025ai}.
\end{theorem}

\cref{thm1} describes the shape of the posterior bootstrap distribution around $\hat\theta_{n,m}^\gamma$. However, the inferential usefulness of this posterior still depends on whether the corresponding mixed oracle target $\theta_{0,m}^\gamma$ is close to $\theta_0$. The next result shows that the leading centering bias is governed by the generalized score discrepancy $(P_0 - P_{\hat\eta_m}) g_{\theta_0}$, which we refer to from now on as score discrepancy. Hence, the goal of rectification is not necessarily to recover $P_0$; rather, it is enough to ensure that 
$P_{\hat\eta_m} g_{\theta_0} \approx P_0 g_{\theta_0} = 0$.

\begin{assumption}\label{cond1} 
Assume the following.
\begin{enumerate}
    \item[(a)] The target parameter is an interior point of the parameter space, i.e., $\theta_0 \in \mathrm{int}(\Theta)$ and satisfies the population score equation:  $P_0g_{\theta_0} = 0$.
    \item[(b)] There exists a ball $B(\theta_0, \delta) \subset \mathrm{int}(\Theta)$ for some $\delta >0$ such that $g_\theta$ is differentiable on $B(\theta_0, \delta)$, $J_0 := P_0 \dot g_{\theta_0}$ is nonsingular, the map $\theta \mapsto P_0 \dot g_\theta$ is continuous at $\theta_0$, and $\sup_{\theta \in B(\theta_0, \delta)} \| (P_{\hat\eta_m} - P_0)\dot g_\theta \| = o_p(1)$.
    \item[(c)] For some deterministic sequence $r_m \to 0$, $ \| (P_{\hat\eta_m} - P_0) g_{\theta_0} \| = O_p(r_m)$.
\end{enumerate}
\end{assumption}

\begin{theorem}\label{thm2}
Suppose \cref{cond1} holds and $\theta_{0,m}^\gamma \xrightarrow[]{\mathrm{p}} \theta_0$ as $n,m \to \infty$. Then 
\begin{equation*}
    \theta_{0,m}^\gamma - \theta_0 = \frac{\gamma}{1 + \gamma} J_0^{-1} (P_0 - P_{\hat \eta_m}) g_{\theta_0} + o_p (r_m).
\end{equation*}
In particular, if $(P_0 - P_{\hat \eta_m}) g_{\theta_0} = O_p(m^{-1/2})$, then $\theta_{0,m}^\gamma - \theta_0 = O_p(m^{-1/2})$.
\end{theorem}

The consistency condition $\theta_{0,m}^\gamma \to\theta_0$ can be verified by \cref{prop1}. \cref{thm2} makes the role of rectification explicit: for fixed $\gamma$, the leading shift of $\theta_{0,m}^{\gamma}$ away from $\theta_0$ is determined by $(P_0-P_{\hat\eta_m})g_{\theta_0}$. The factor $J_0^{-1}$ translates this into parameter bias, while $\gamma/(1+\gamma)$ reflects the strength of borrowing from the rectified base measure. When $m \asymp n$, the bias is generally on the same scale as the posterior bootstrap fluctuation in \cref{thm1}. Thus, root-$m$ rectification improves the posterior center, but a $\theta_0$-centered first-order limit would require a stronger condition such as $ (P_{\hat \eta_m} - P_0) g_{\theta_0} = o_p(n^{-1/2})$, or an additional bias-correction step.

\section{Rectifier Taxonomy}\label{sec:rectifier-taxonomy}
In this section, we discuss how to construct the rectified base measure $P_{\hat\eta_m}$ in practice. \cref{thm2} shows that the leading bias of the rectified target depends on $(P_{\hat \eta_m} - P_0) g_{\theta_0}$. Thus, a good rectifier should make the AI-induced estimating equation look more like that under the true data-generating law. This does not necessarily require learning the full distribution $P_0$. For a given downstream loss, it may be enough to correct the features of the AI base measure that enter the corresponding score.

There are two natural ways to pursue this goal. 
The first is generic rectification, which is conceptually simple and can provide stable corrections to the AI-induced law. For univariate outcomes, our default recommendation is \emph{quantile mapping}, which transports the marginal distribution of AI-imputed outcomes towards that of the calibration data while preserving their rank-ordering. Specifically, we apply the transformation $y\mapsto \hat{F}_Y^{-1} \big(\hat{F}_{\hat{Y}}(y)\big)$ to the AI-imputed outcomes, where $\hat{F}_{\hat{Y}}$ and $\hat{F}_Y^{-1}$ are the empirical distribution function and empirical quantile function of the AI-imputed and gold-standard outcomes on the calibration sample. Quantile mapping is highly effective when the main source of error is a monotone marginal miscalibration, in which case it will substantially reduce score mismatch for downstream estimands.

The second approach is targeted rectification, where the rectifier is chosen to directly align the score with respect to a particular loss with the goal of minimizing $(P_{\hat \eta_m} - P_0) g_{\theta_0}$. In practice, since $\theta_0$ is unknown, we can use a plug-in estimate $\tilde \theta$ and try to fit $\hat \eta_m$ so that $P_{\hat \eta_m} g_{\tilde \theta}$ matches $P_m g_{\tilde \theta}$. Here, $P_m$ denotes the empirical distribution of the labeled calibration sample used to fit the rectifier. If this fitting procedure yields $(P_{\hat \eta_m} - P_0) g_{\theta_0} = O_p(m^{-1/2})$, then the centering bias of the rectified target is $O_p(m^{-1/2})$. Targeted rectification is particularly convenient when the score moments $P_{\eta} g_{\theta_0}$ depend on the law $P_{\eta}$ only through a finite-dimensional moment $P_\eta h$, where score-matching reduces to finite moment matching. The following proposition gives a simple sufficient condition under which such moment matching controls the leading score discrepancy.

\begin{proposition}\label{prop:moment-matching}
Suppose that there exists $h:\mathcal{X}\times\mathcal{Y}\to\mathbb{R}^k$ and $\Gamma:\mathbb{R}^k\to\mathbb{R}^{d}$ such that $P g_{\theta_0}=\Gamma \left(P h\right)$ for every $P\in\{P_0\}\cup\{P_\eta:\eta\in\mathcal{H}\}$ and $P_0 \|h\|^2<\infty$. Suppose also that $\Gamma$ is  Lipschitz, or that it is locally Lipschitz within a neighborhood of $P_0h$ and the fitted rectifier satisfies $P_{\hat{\eta}_m}h-P_0h=o_p(1)$. Given these conditions, as long as the fitted rectifier satisfies $P_{\hat{\eta}_m}h-P_mh=O_p(m^{-1/2})$, then $(P_{\hat{\eta}_m}-P_0)g_{\theta_0}=O_p(m^{-1/2})$.  
\end{proposition}

\cref{prop:moment-matching} shows that matching the finite-dimensional empirical moment $P_mh$ is sufficient to control the leading score discrepancy. The sufficient conditions given are also related to the exponential family, as whenever the loss $\ell_\theta$ is the kernel of an exponential family likelihood, the score $g_\theta$ is affine in the sufficient statistic, implying that such a function $h$ exists and that $\Gamma$ is globally Lipschitz. Below, we discuss some specific examples. 

\begin{example} \textbf{(Mean estimation).}
Let $\ell_\theta(x,y)=(y-\theta)^2/2$. Then $g_\theta(y)=\theta-y$ so at $\theta_0$, $Pg_\theta=\theta_0-PY$. We take $h(y)=y$, $\Gamma(u)=\theta_0-u$. In this case, score matching reduces to mean matching, where rectification reduces to ensuring that $P_{\hat{\eta}_m}Y$ is close to the empirical mean $P_mY$.
\end{example}
\begin{example}\textbf{(Linear regression).}
Let $X\in\mathbb{R}^p$ and $\ell_\theta(x,y)=(y-x^\top \theta)^2/2$. Then $g_\theta(x,y)=x(x^\top\theta-y)$, and at $\theta_0$, we have $Pg_{\theta_0}=P(xx^\top)\theta_0-P(xy)$. We take $h(x,y)=(\mathrm{vec}(xx^\top), xy)$ and $\Gamma(u,v)=\mathrm{Mat}(u)\theta_0-v$, where for a $p\times p$ matrix $A$, $\mathrm{vec}(A)\in\mathbb{R}^{p^2}$ denotes the vectorization of $A$ and $\mathrm{Mat}(\mathrm{vec}(A))=A$. The rectifier needs to match these two moments. If the base measure preserves the marginal law of $X$, the fitting procedure further reduces to only matching $P_m(XY)$.
\end{example}

\textbf{Practical guidelines.} For univariate outcomes, we recommend practitioners first reach for quantile mapping due to its simplicity and applicability in correcting a wide range of common miscalibrations.  Another concern is constructing the labeled calibration sample used to fit the rectifier. Options include using the full labeled sample (adjacent to empirical Bayes), sample splitting, or drawing a nonparametric bootstrap sample from the labeled sample. We recommend the last approach as it propagates uncertainty in estimating the rectifier into the posterior bootstrap distribution, while avoiding the efficiency loss of sample splitting.

\section{Related Work}\label{sec:related-work}
A growing literature has explored using LLMs and other foundation models for prior elicitation.  \citet{riegler2025using} prompt LLMs to propose parametric priors on regression coefficients. Similarly, \citet{arai2025patients} consider a hierarchical model for adverse events in multi-center clinical trials, and prompt LLMs to elicit rate parameters of hyperpriors on the parameters of gamma distributed site-specific adverse effect rates. Unlike these approaches that elicit parametric priors directly by prompting LLMs, \citet{misra2025foundation} treats LLM outputs as draws from a subjective prior predictive, to be incorporated into the user's own prior via a generalized Bayesian update. \citet{o2025ai} describe priors based on synthetic data generated by an AI-induced law. Our contribution is centered around correcting synthetic data-generation approaches to AI-based prior construction.

Using external data to construct or calibrate a prior distribution is well-established in Bayesian statistics, reflecting the view that  inference should combine data with relevant prior knowledge. \citet{chen2000power} incorporate historical data into the prior through a possibly discounted likelihood factor. The expected-posterior prior~\citep{perez2002expected} averages posteriors obtained from imaginary training samples drawn from a reference distribution. More recently, \citet{huang2025catalytic} construct a prior for a complex target model using synthetic data generated from a simpler donor model. Our approach is similar in spirit but allows for downstream inferential bias to be controlled.

\emph{Inference with predicted data} (IPD) is an area of growing interest~\citep{salerno2025really} in which covariates are fully observed while the outcome variable is only partially observed, and the practitioner uses predicted outcomes provided by a machine learning model. The naive incorporation of these predicted outcomes as additional data can cause error to propagate and invalidate the downstream inferential task \citep{ogburn2021warning}. Post-prediction inference~\citep{wang2020methods} involves modeling the relationship between real data and predicted outcomes and using the modeled relationship to correct predicted outcomes. Prediction-powered inference~\citep{angelopoulos2023prediction} (PPI) provides a principled means of valid statistical inference for certain types of parameters in this setting, including risk minimizers corresponding to convex loss functions. \citet{angelopoulos2024ppiplusplus} extend PPI and provide computational and statistical efficiency improvements. \citet{zrnic2024cross} consider when labels are imputed from a predictor trained on the labeled data itself. Both post-prediction inference and prediction-powered inference apply a correction at the level of the downstream estimator or inferential procedure, while we learn a rectified AI-induced probability law which has relevant moments aligned with the true data-generating law. The IPD setting is addressed by a special case of DP-based synthetic data priors when the base measure is an empirical measure on pairs of unlabeled covariates and deterministically imputed outcomes.

The use of predicted labels is prominent in computational social science as expensive human annotations can be easily imputed with LLMs. \citet{egami2023using} propose a framework for ensuring the validity of downstream statistical analyses when augmenting a set of human-labeled data with additional, potentially invalid LLM-labeled data. \citet{gligoric2025unconfident} further incorporate LLM-provided confidence indicators to reduce the number of human annotations that must be collected.

\section{Experiments}\label{sec:experiments}
In this section, we construct and apply the rectified AI prior in real-world applications. We evaluate the performance of the rectified AI prior in two settings. First, we study inferential targets like quantiles or regression coefficients, where we focus on the calibration and width of posterior credible intervals. Next, we study a downstream classification task where the goal is predictive performance of a Bayesian deep learner equipped with an LLM-informed prior. These settings highlight the broad benefits of rectification: both improved frequentist validity of inference, as well as more reliable incorporation of AI-generated prior information in downstream learning tasks.

\subsection{Statistical Inference on Population Risk Minimizers}\label{sec:exp-inference}

\textbf{Baselines.} We compare with the following approaches: the classical method based on labeled data only, prediction-powered inference (PPI++), and the AI prior centered at the raw AI base measure.

\textbf{Experimental setup.} For each experiment, we consider the labeled sample size to take values $n\in\{250,500,1000\}$. For each AI prior variation, we choose prior strength $\gamma=1$ and draw $B=500$ posterior bootstrap samples. Credible intervals are computed using empirical quantiles of these samples. For each method and $n$ value, we perform 100 replications from which we estimate performance metrics. Each replication has a different random subset of $n$ indices to be treated as the labeled sample, and reflects stochasticity in the labeled sample as well as any randomness in each interval construction procedure. We use a quantile map rectifier that is re-fit on an independent nonparametric bootstrap sample of the labeled data for each posterior bootstrap iteration. 

\textbf{Metrics.} We evaluate each method by the empirical coverages of $90\%$ intervals, as well as the interval score~\citep{gneiting2007strictly}. For a $1-\beta$ interval $(L,U)$ and target value $\theta_0$, the interval score is given by $(U-L)+2\beta^{-1}(L-\theta_0)\mathbb{I}\{\theta_0<L\}+2\beta^{-1}(\theta_0-U)\mathbb{I}\{\theta_0>U\}$. This is a proper scoring rule which rewards narrow intervals, penalizes failure to cover the truth, and is minimized in expectation by the central $1-\beta$ interval of the true distribution.

\begin{figure}[t!]
    \centering
    \includegraphics[width=\linewidth]{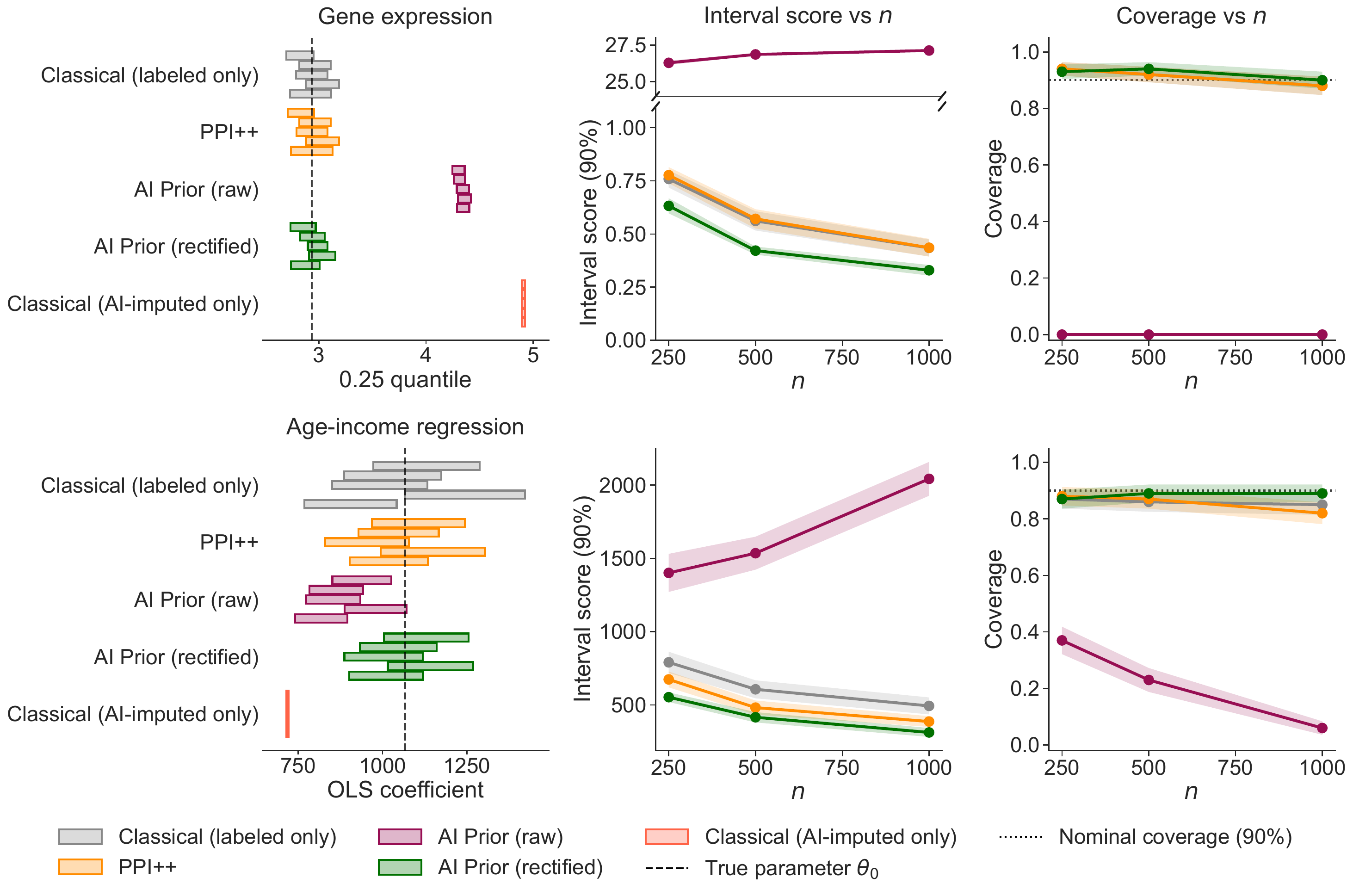}
    \caption{Interval estimates for population risk minimizers in gene expression and age-income regression experiments. We visualize 5 realizations of intervals, and compare the average interval score and coverage across 100 replications. Rectifying the prior yields intervals that achieve nominal coverage and have low interval scores. Error bars indicate one standard error.}
    \label{fig:results-gene-census}
\end{figure}

\textbf{Gene expression levels.} We apply our method to do statistical inference on quantiles of the expression level of a gene induced by native yeast promoter sequences. We follow the setup of previous work~\citep{angelopoulos2023prediction, o2025ai} which found that leveraging a transformer model to impute unknown expression levels can aid in statistical inference on the quantiles. \citet{vaishnav2022evolution} provide both the dataset and transformer model for label imputation. The data consist of $61,150$ native yeast promoter sequences, while the response $Y$ is the measured gene expression level. The parameter of interest is the fixed quantile $\tau=0.25$ of $Y$, which corresponds to the population minimizer of the $\tau$-check loss. We impute the real-valued gene expression levels for unlabeled data using the transformer model. The raw AI prior uses the empirical base measure induced by the unlabeled sample and the transformer. 

The first row of~\cref{fig:results-gene-census} displays some of the resulting interval estimates and the average interval scores and coverages (and one standard error) over 100 replications. When using $\gamma=1$, the AI base measure upwardly biases the resulting credible intervals for the 0.25 quantile of gene expression. The dramatic amount of bias injected from putting too much trust in the prior causes the credible intervals to have near zero coverage, and the resulting posterior is nearly useless for estimation and inference. Rectification is able to remedy this, causing the resulting interval score and coverage to be competitive with the PPI++ confidence interval.

\textbf{Age-income regression.} The goal is to investigate the relationship between age and income in the US Census data. We follow the census-income experiment of \citet{angelopoulos2023prediction} and use the same prediction setup. Specifically, the data are extracted from the California 2019 census via the Folktables interface and consist of $378,817$ individuals. The response $Y$ is yearly income (in dollars), and the covariate $X$ contains age and sex. The parameter of interest is the ordinary least squares coefficient of age in the population. Following \citet{angelopoulos2023prediction}, we use income predictions from a gradient-boosted tree trained on the previous year's raw census data. The raw AI prior uses the unlabeled sample and corresponding AI-predicted incomes as the empirical base measure. 

The second row in \cref{fig:results-gene-census} summarizes the resulting credible intervals and average interval score and coverage (and one standard error) over 100 replications. The raw AI prior produces intervals that are visibly shifted away from the target value, and both its interval score and empirical coverage become worse as $n$ increases. In contrast, the rectified AI prior corrects much of this shift---its intervals are centered closer to the target value, its interval score is comparable to those of the classical and PPI++ baselines, and its empirical coverage remains close to the nominal $90\%$ level. 

\textbf{Ablations and additional experiments.}  In~\cref{sec:supp-inf-experiments}, we include ablations with alternate specifications. \cref{fig:supp-gamma-ablation} considers different values of the prior strength $\gamma$. \cref{fig:supp-gene-quantile-ablation} includes results for alternative quantiles in the gene expression experiment. We also provide a comparison of different rectifiers and strategies for selecting the calibration sample used to fit the rectifier in~\cref{tab:supp-rectifier-ablation}.

\subsection{Deep Learning Classification of Erythemato-Squamous Disease}\label{sec:exp-skin}

\begin{figure}[t!]
    \centering
    \includegraphics[width=0.9\linewidth]{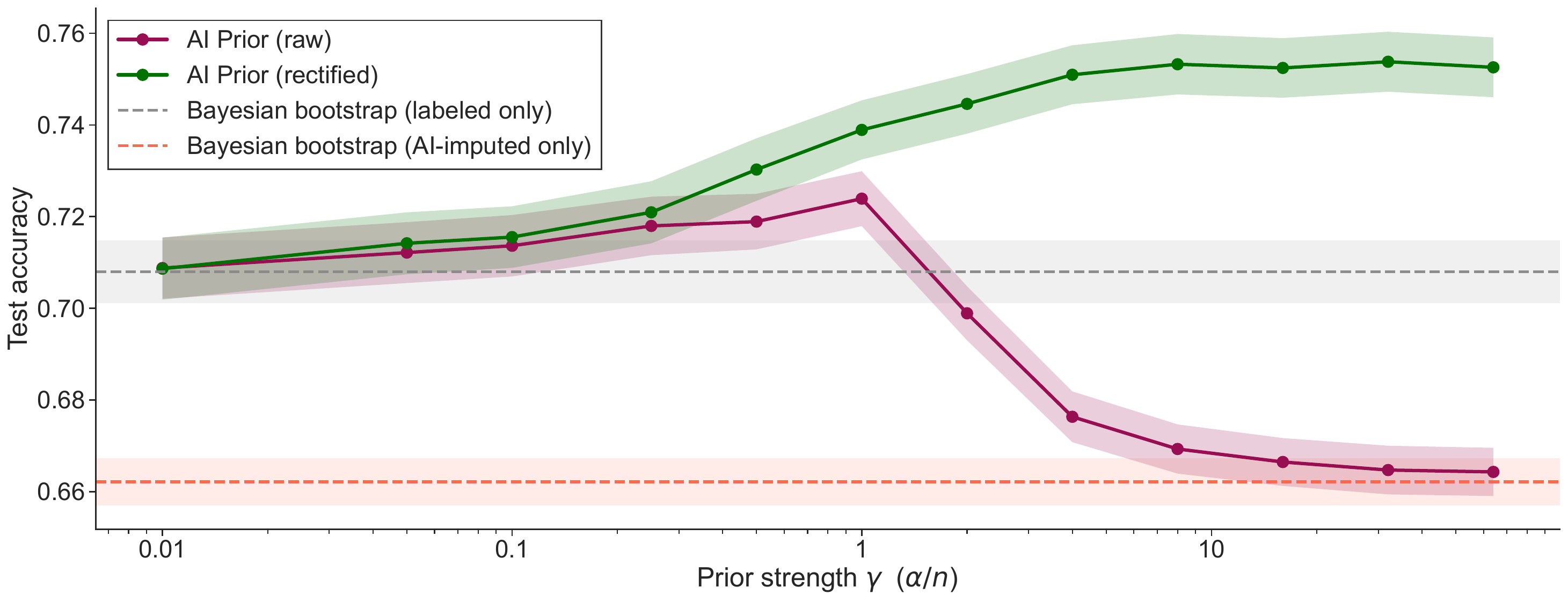}
    \caption{Average test accuracy as a function of prior strength over 100 repetitions when using raw and rectified AI priors on skin-disease classification. Rectifying the base measure allows for a larger gain in classification performance, which is available across a wide range of values for the prior strength. Error bars indicate one standard error.}
    \label{fig:results-skin}
\end{figure}

We also study a skin disease classification problem adapted from the UCI dermatology dataset~\citep{dermatology_33}. The dataset consists of 366 patients, and contains 12 clinically observable features alongside 22 histopathological features. We follow the experimental setup of~\citet{o2025ai} where we seek to classify the type of erythemato-squamous disease among six possible categories using only the 12 clinical features, which are observable without requiring biopsy. We consider 73 observations (20\% of the data) to be held-out test data, and $n=58$ observations to be labeled (20\% of the remaining data). The 235 leftover patients are considered to be unlabeled. For each unlabeled patient $X$, we impute class probabilities $\hat{p}$ by prompting an LLM. We used the GPT 5.5 model, and details regarding the prompt, output schema, and full structure of the AI base measure are available in~\cref{sec:skin-disease-base-measure-supp}. The parameter of interest is the weights of a multilayer perceptron under the cross-entropy loss. 

The AI prior uses the empirical base measure induced by covariates $X$ in the unlabeled set and imputed labels sampled from a categorical distribution of the LLM-imputed class probabilities $\hat{p}$. We rectify the AI prior by replacing $\hat{p}$ with $g(\hat{p}, X)$, where $g$ is a multinomial logistic regression on the log imputed probabilities and covariates. Explicitly, $g(\hat{p},X) = \mathrm{softmax}(W (\log \hat{p},X)+b)$, where $W\in\mathbb{R}^{6\times 18}$ and $b\in\mathbb{R}^6$, and $(W,b)$ are fit via cross-entropy loss on a nonparametric bootstrap sample. This is related to post-hoc probability calibration~\citep{guo2017calibration, kull2019beyond}.

We conduct our experiment over a grid of values for the prior strength parameter $\gamma$. For each value, we draw $B=100$ samples using the posterior bootstrap~\citep{fong2019scalable}. For each posterior bootstrap sample, we first fit the rectifier separately by drawing a calibration sample from the empirical distribution of the labeled data, which we use to fit the rectifier. The posterior bootstrap sample is then drawn by approximately solving the optimization problem against the Dirichlet weighted cross-entropy loss. The accuracy on the test set is computed by averaging the predicted class probabilities across the posterior samples and classifying via the argmax.

Both methods are compared against two baselines, which represent limiting cases of the standard AI prior. As the prior strength $\gamma$ tends towards zero, the method acts like the standard Bayesian bootstrap~\citep{rubin1981bayesian, lyddon2019general} fit on labeled data only. This is effectively the baseline performance of the multilayer perceptron. As $\gamma$ grows large, the influence of the imputed data dominates that of the labeled data, and the method acts like the Bayesian bootstrap multilayer perceptron performance when fit on imputed data only.  

\cref{fig:results-skin} displays the average test accuracy (and one standard error) over 100 repetitions with independent partitions of the indices of the dataset into the test, labeled, and unlabeled sets. We find a relatively small window of prior strength values where the raw AI prior yields a small increase in classification performance. However, calibrating the prior strength $\gamma$ correctly within this region would require the practitioner to have additional held-out validation data. Rectifying the AI prior substantially reduces the sensitivity to $\gamma$ in this experiment, providing a performance increase over the labeled-only Bayesian bootstrap across the entire tested region for $\gamma$, from 0.01 to 64. In addition, the increase in test accuracy grew substantially larger than the peak of the raw AI prior, reflecting an additive increase in classification accuracy of 0.049 compared to just 0.015.

This experiment highlights that the benefits of rectification extend beyond the setting of correcting centering bias in univariate estimation tasks. Through rectification, we are able to leverage an LLM in the data-scarce setting to create a substantially more performant classifier for erythemato-squamous disease, with large performance increases across a wide range of the prior strength $\gamma$.

\section{Conclusion}\label{sec:conclusion}
We proposed a rectified AI-powered prior that uses labeled calibration data to transform the raw AI base measure into a rectified base measure before centering a DP prior. This makes modern predictive systems usable for informative prior knowledge, while shifting the AI prior concentration parameter from a calibration device for frequentist reliability to a measure of prior strength. We developed theory showing that rectification can preserve the posterior bootstrap limit of AI-powered Bayesian inference and make the leading centering bias explicit as the task-relevant score discrepancy. Positive broader impacts of our work include improved statistical efficiency in data-limited applications. Biased or overconfident conclusions could arise if the rectifier is misspecified. This method should be used with transparent reporting, especially in high-stakes applications.

\textbf{Limitations and future work.} In general, the choice of the rectifier is problem-dependent and may be non-obvious for applied practitioners. The quantile mapping rectifier is only applicable for univariate outcome variables, and is only capable of correcting certain classes of miscalibration. A useful direction is rectifier model selection: using the labeled calibration data to choose among candidate rectifier families based on their performance on the target inferential task. In order to incorporate uncertainty from estimation of the rectifier into the resulting posterior, we consider re-fitting it each time from a re-randomized sample split or from a nonparametric bootstrap sample of the labeled data. Future work may investigate other avenues to propagate rectifier uncertainty into the posterior, such as a fully Bayes approach.

\renewcommand{\refname}{References}
\putbib
\end{bibunit}

\clearpage
\begin{bibunit}

\appendix

\setcounter{assumption}{0}
\renewcommand{\theassumption}{S\arabic{assumption}}

\section{Proofs}\label{supp:proofs}
In this section, we provide proofs for the results discussed in the main text.

The following conditions are sufficient for \cref{thm1}. All $o_p(\cdot)$ and $O_p(\cdot)$ statements involving $P_{\hat\eta_m}$ are with respect to the joint randomness of the calibration and inference samples, unless otherwise stated. Conditional statements are understood conditionally on the calibration sample used to construct $P_{\hat{\eta}_m}$. 

\begin{assumption}\label{supp:cond2} 
Assume the following
\begin{enumerate}
    \item[(a)] $\Theta \subset \mathbb R^d$ is compact and convex. The loss function $\ell: \Theta \times \mathcal X \times \mathcal Y \to \mathbb R$ is measurable and bounded from below. For every $\theta \in \Theta$, $P_0 \ell_\theta < \infty$ and $P_{\hat \eta_m} \ell_\theta < \infty$.

    \item[(b)] For each fixed $\gamma >0$ and every $\varepsilon >0$,
    \begin{equation*}
        \inf_{\theta \in \Theta: \|\theta - \theta_{0,m}^\gamma\| \ge \varepsilon} \Big(P_0 \ell_\theta + \gamma P_{\hat\eta_m}\ell_\theta - \big(P_0\ell_{\theta_{0,m}^\gamma} + \gamma P_{\hat\eta_m} \ell_{\theta_{0,m}^\gamma} \big) \Big) >0,
    \end{equation*}
    with probability tending to one. Moreover, $\sup_{\theta \in \Theta} \big\lvert P_n \ell_\theta - P_0 \ell_{\theta} \big\rvert = o_p(1)$.
    
    \item[(c)] For each fixed $\gamma >0$, conditional on the fitted base measure $P_{\hat \eta_m}$, there exists an open ball $B_m$ with $\theta_{0,m}^\gamma \in B_m \subset \mathrm{int}(\Theta)$, and for $P_0-$ and $P_{\hat \eta_m}-$ almost every $(X,Y)$, the first three partial derivatives of $\ell_\theta(X, Y)$ with respect to $\theta$ exist and are continuous for all $\theta \in B_m$. For $j,k,l \in \{1,\ldots, d\}$, there exist measurable functions $G_j$, $G_{jk}$, $G_{jkl}$, $H_{jkl}$, possibly depending on the calibration sample, such that for all $\theta \in B_m$,
    \begin{align*}
        \bigg\lvert \frac{\partial}{\partial \theta_j} \ell_\theta (X, Y) \bigg\rvert & \le G_j(X, Y), \\
        \bigg\lvert \frac{\partial^2}{\partial \theta_j \partial \theta_k} \ell_\theta (X, Y) \bigg\rvert & \le G_{jk}(X, Y), \\
        \bigg\lvert \frac{\partial^3}{\partial \theta_j \partial \theta_k \partial \theta_l} \ell_\theta (X, Y) \bigg\rvert & \le G_{jkl}(X, Y), \\
    \end{align*}
    and
    \begin{equation*}
        \bigg\lvert \frac{\partial}{\partial \theta_j} \ell_\theta(X,Y) \frac{\partial^2}{\partial \theta_k \partial \theta_l} \ell_\theta (X, Y) \bigg\rvert \le H_{jkl}(X, Y).
    \end{equation*}
    Moreover, these functions are integrable under both $P_0$ and $P_{\hat \eta_m}$: $P_0 G_j + P_{\hat \eta_m}G_j < \infty$, $P_0 G_{jk} + P_{\hat \eta_m}G_{jk} < \infty$, $P_0 G_{jkl} + P_{\hat \eta_m}G_{jkl} < \infty$, and $P_0 H_{jkl} + P_{\hat \eta_m} H_{jkl} < \infty$.

    \item[(d)] For $\theta \in B_m$, define $J_1(\theta) := P_0 \dot g_\theta$, $J_{2,m}(\theta) := P_{\hat\eta_m} \dot g_\theta$, and $I_1(\theta) := P_0(g_\theta g_\theta^\mathtt{T})$, $I_{2,m}(\theta) := P_{\hat\eta_m}(g_\theta g_\theta^\mathtt{T})$. Assume all of these matrices are finite for all $\theta \in B_m$. Define
    \begin{equation*}
        J_m(\theta) = \frac{1}{1 + \gamma} \big(J_1(\theta) + \gamma J_{2,m}(\theta) \big),
    \end{equation*}
    and
    \begin{equation*}
        I_m(\theta) = \frac{1}{1 + \gamma} \big(I_1(\theta) + \gamma I_{2,m}(\theta) \big).
    \end{equation*}
    Assume $J_m(\theta_{0,m}^\gamma)$ is nonsingular with $ \| J_m(\theta_{0,m}^\gamma)^{-1} \| = O_p(1)$, and $I_m(\theta_{0,m}^\gamma)$ is finite and positive definite.
\end{enumerate}
\end{assumption}

\cref{supp:cond2} (c)-(d) are the rectified versions of the regularity conditions used in \citet{o2025ai}. Conditional on $P_{\hat \eta_m}$, they play the role of the corresponding smoothness and information conditions with $F_{AI}$ replaced by $P_{\hat{\eta}_m}$. 

\begin{proposition}\label{prop1}
    Suppose \cref{supp:cond2}(a) holds. In addition, assume that $\theta_0$ is a well-separated minimizer of $P_0\ell_\theta$, i.e., for every $\varepsilon >0$, $\inf_{\theta \in \Theta: \|\theta-\theta_0\| \ge \varepsilon} \big(P_0\ell_\theta - P_0 \ell_{\theta_0} \big) >0$. Assume also that $\sup_{\theta \in \Theta} \big|P_n\ell_\theta - P_0 \ell_\theta \big| = o_p(1)$, and $\sup_{\theta \in \Theta} \big| P_{\hat\eta_m}\ell_\theta - P_0\ell_\theta \big| = o_p(1)$. 
    Then, as $n,m\to\infty$,
    \begin{equation*}
        \theta_{0,m}^{\gamma} \xrightarrow[]{\mathrm{p}} \theta_0, \qquad \hat{\theta}_{n,m}^\gamma \xrightarrow[]{\mathrm{p}} \theta_0.
    \end{equation*}
\end{proposition}

\begin{proof}[Proof of \cref{prop1}]
    By \cref{supp:cond2}(a), all risks below are well-defined and finite for each $\theta \in \Theta$. Define $R^\gamma(\theta) := (1+\gamma)P_0 \ell_\theta$. By assumption, $\theta_0$ is a well-separated minimizer of $P_0\ell_\theta$, hence also of $R^\gamma$. That is, for every $\epsilon >0$, $\inf_{\theta: \|\theta - \theta_0\| \ge \epsilon} \big( R^\gamma(\theta) - R^\gamma(\theta_0) \big) > 0$. First consider the oracle mixed risk $R_{0,m}^\gamma(\theta) := P_0 \ell_\theta + \gamma P_{\hat\eta_m}\ell_\theta$. Then $R_{0,m}^\gamma(\theta) - R^\gamma(\theta) = \gamma \big( P_{\hat\eta_m}\ell_\theta - P_0\ell_\theta \big)$. Therefore,
    \begin{equation}
        \sup_{\theta \in \Theta} \big\lvert R_{0,m}^\gamma(\theta) - R^\gamma(\theta) \big\rvert \le \gamma \sup_{\theta \in \Theta} \big\lvert P_{\hat\eta_m}\ell_\theta - P_0\ell_\theta \big\rvert = o_p(1).
    \end{equation}
    
    We apply Theorem 5.7 of \citealp{Vaart_1998} to the negative risks. Set $M_m(\theta) = -R_{0,m}^\gamma (\theta)$ and $M(\theta) = -R^\gamma(\theta)$. Then, the uniform convergence condition of Theorem 5.7 is satisfied because $\sup_{\theta \in \Theta} \big\lvert M_m(\theta) - M(\theta) \big\rvert = o_p(1)$. Because $R^\gamma$ has a well-separated minimizer at $\theta_0$, $M$ has a well-separated maximizer at $\theta_0$. Also, since $\theta_{0,m}^\gamma$ minimizes $R_{0,m}^\gamma$, it maximizes $M_m$. Hence, $M_m(\theta_{0,m}^\gamma) \ge M_m(\theta_0)$. This is stronger than the condition in Theorem 5.7. Therefore, $\theta_{0,m}^\gamma \xrightarrow[]{\mathrm{p}} \theta_0$ as $n,m \to \infty$.

    Next consider the empirical mixed risk $R_{n,m}^\gamma(\theta) = P_n \ell_\theta + \gamma P_{\hat \eta_m} \ell_\theta$. We have 
    \begin{equation}
        R_{n,m}^\gamma(\theta) - R^\gamma(\theta) = \big(P_n\ell_\theta - P_0\ell_\theta\big) + \gamma \big(P_{\hat\eta_m} \ell_\theta - P_0\ell_\theta \big).
    \end{equation}
    Thus
    \begin{equation}
        \sup_{\theta \in \Theta} \big\lvert R_{n,m}^\gamma(\theta) - R^\gamma(\theta) \big\rvert \le \sup_{\theta \in \Theta} \big\lvert P_n\ell_\theta - P_0\ell_\theta \big\rvert + \gamma \sup_{\theta \in \Theta} \big\lvert P_{\hat\eta_m}\ell_\theta - P_0\ell_\theta \big\rvert = o_p(1).
    \end{equation}

    Apply Theorem 5.7 of \citet{Vaart_1998} once again with $M_{n,m}(\theta) = -R_{n,m}^\gamma(\theta)$, $M(\theta) = -R^\gamma(\theta)$. We have $\sup_{\theta \in \Theta} \big\lvert M_{n,m}(\theta) - M(\theta) \big\rvert = o_p(1)$. Also, $M$ has a well-separated maximizer at $\theta_0$, and since $\hat\theta_{n,m}^\gamma$ minimizes $R_{n,m}^\gamma$, it maximizes $M_{n,m}$. Hence $M_{n,m}(\hat\theta_{n,m}^\gamma) \ge M_{n,m}(\theta_0)$, so the condition in Theorem 5.7 holds. Therefore, $\hat \theta_{n,m}^\gamma \xrightarrow[]{\mathrm{p}} \theta_0$ as $n,m \to \infty$. This proves \cref{prop1}.
\end{proof}

\begin{proof}[Proof of \cref{thm1}]
    We condition on the calibration sample used to construct $P_{\hat\eta_m}$. Then $P_{\hat\eta_m}$ is fixed and independent of the inference sample. Therefore, Theorem 1 of \citet{o2025ai} can be applied with $F_{AI}$ replaced by $P_{\hat\eta_m}$. The relevant correspondences are $F_{AI} \leftrightarrow P_{\hat\eta_m}$, $\theta_0^\gamma \leftrightarrow \theta_{0,m}^\gamma$, $\hat\theta_n^\alpha \leftrightarrow \hat\theta_{n,m}^\gamma$, and $\alpha = \gamma n$. Hence the total posterior bootstrap mass is $n + \alpha = n (1 + \gamma)$, which gives the scaling $\sqrt{n(1 + \gamma)}$. 

    The empirical mixed objective in our setting is $\theta \mapsto P_n \ell_\theta + \gamma P_{\hat \eta_m} \ell_\theta$ and its minimizer is $\hat \theta_{n,m}^\gamma$. The corresponding mixed population objective is $\theta \mapsto P_0\ell_\theta + \gamma P_{\hat \eta_m} \ell_\theta$ and its minimizer is $\theta_{0,m}^\gamma$. Now define the normalized mixed population measure
    \begin{equation}
        Q_m^\gamma = \frac{1}{1 + \gamma} P_0 + \frac{\gamma}{1 + \gamma} P_{\hat\eta_m}.
    \end{equation}
    Then
    \begin{equation}
        P_0 \ell_\theta + \gamma P_{\hat \eta_m}\ell_\theta = (1 + \gamma) Q_m^\gamma \ell_\theta,
    \end{equation}
    so $\theta_{0,m}^\gamma$ is also the minimizer of $Q_m^\gamma \ell_\theta$. The associated Hessian and score second moment matrices are
    \begin{equation*}
        J_m = \frac{1}{1 + \gamma} \big(J_1(\theta_{0,m}^\gamma) + \gamma J_{2,m}(\theta_{0,m}^\gamma) \big),
    \end{equation*}
    and
    \begin{equation*}
        I_m = \frac{1}{1 + \gamma} \big(I_1(\theta_{0,m}^\gamma) + \gamma I_{2,m}(\theta_{0,m}^\gamma) \big).
    \end{equation*}
    
    By \cref{supp:cond2}, conditional on the calibration sample used to construct $P_{\hat\eta_m}$, the regularity conditions required for Theorem 1 of \citet{o2025ai} hold with $P_{\hat\eta_m}$ replacing $F_{AI}$. Applying their Theorem 1 conditionally on the calibration sample, we obtain
    \begin{equation}
        \Pi_{\mathrm{PB}}^{\hat{\eta}_m} \Big[ \sqrt{n(1+\gamma)} \big( \theta - \hat{\theta}_{n,m}^\gamma \big) \in A \Big] \xrightarrow[]{\mathrm{p}} P(Z_m \in A \mid P_{\hat \eta_m})
    \end{equation}
    for every Borel set $A \subset \mathbb R^d$ satisfying $P(Z_m \in \partial A \mid P_{\hat \eta_m}) =0$ where $Z_m \mid P_{\hat \eta_m} \sim \mathcal N(0, J_m^{-1}I_m J_m^{-1})$. The matrices $J_m$ and $I_m$ coincide with the $J$ and $I$ matrices in that theorem. Their proof given in Appendix A.2 follows the loss-likelihood bootstrap argument of \citep{lyddon2019general, Newton1991}. This proves the claim. 
\end{proof}

\begin{proof}[Proof of \cref{thm2}]
    Let $d_m := \theta_{0,m}^\gamma - \theta_0$. By the assumed consistency, $d_m = o_p(1)$. Since $\theta_0 \in \mathrm{int}(\Theta)$ and $B(\theta_0, \delta) \subset \mathrm{int}(\Theta)$, we have $\theta_{0,m}^\gamma \in B(\theta_0, \delta)$ with probability tending to one. Therefore, we have 
    \begin{equation}\label{eq:first-order}
        P_0 g_{\theta_{0,m}^\gamma} + \gamma P_{\hat\eta_m} g_{\theta_{0,m}^\gamma} = 0.
    \end{equation}
    By \cref{cond1}(b), $g_\theta$ is differentiable near $\theta_0$ and the map $\theta \mapsto P_0 \dot g_{\theta}$ is continuous at $\theta_0$. Then by Taylor expansion around $\theta_0$, and by \cref{cond1}(b),
    \begin{equation}\label{eq:taylor1}
        P_0 g_{\theta_{0,m}^\gamma} = P_0 g_{\theta_0} + \big( J_0 + o_p(1) \big) d_m.
    \end{equation}
    Similarly,
    \begin{equation}\label{eq:taylor2}
        P_{\hat\eta_m} g_{\theta_{0,m}^\gamma} = P_{\hat\eta_m} g_{\theta_{0}} + \big(J_0 + o_p(1) \big) d_m. 
    \end{equation}
    For this second expansion we use \cref{cond1}(b) again. Indeed, locally around $\theta_0$,
    \begin{equation}
        P_{\hat\eta_m} \dot g_{\theta} = P_0 \dot g_\theta + (P_{\hat\eta_m} - P_0 ) \dot g_\theta.
    \end{equation}
    Hence $P_{\hat\eta_m} \dot g_\theta = J_0 + o_p(1)$ along the segment between $\theta_0$ and $\theta_{0,m}^\gamma$.

    Substituting \cref{eq:taylor1} and \cref{eq:taylor2} into \cref{eq:first-order} gives
    \begin{equation}
        0 = P_0 g_{\theta_0} + \gamma P_{\hat\eta_m}g_{\theta_0} + \big( (1+\gamma)J_0 + o_p(1) \big) d_m.
    \end{equation}
    Since $P_0 g_{\theta_0} =0$ by \cref{cond1}(a), $P_{\hat\eta_m} g_{\theta_0} = (P_{\hat\eta_m} - P_0) g_{\theta_0}$. Thus, $\big( (1+\gamma)J_0 + o_p(1) \big) d_m = \gamma (P_0 - P_{\hat\eta_m} ) g_{\theta_0}$. By \cref{cond1}(b), $J_0$ is nonsingular. Hence
    \begin{equation}
        d_m = \frac{\gamma}{1 + \gamma} J_0^{-1} (P_0 - P_{\hat\eta_m}) g_{\theta_0} + R_m,
    \end{equation}
    where $\|R_m\| \le o_p(1)  \, \| (P_0 - P_{\hat\eta_m}) g_{\theta_0} \|$. But, by \cref{cond1}(c), we have $R_m = o_p(r_m)$. Therefore,
    \begin{equation}
        \theta_{0,m}^\gamma - \theta_0 = \frac{\gamma}{1 + \gamma} J_0^{-1} (P_0 - P_{\hat \eta_m}) g_{\theta_0} + o_p (r_m).
    \end{equation}
    If $(P_0 - P_{\hat\eta_m}) g_{\theta_0} = O_p(m^{-1/2})$, then take $r_m = m^{-1/2}$, and we get $\theta_{0,m}^\gamma - \theta_0 = O_p(m^{-1/2})$.
\end{proof}

\begin{proof}[Proof of \cref{prop:moment-matching}]
If $\Gamma$ is only locally Lipschitz, then since $P_{\hat{\eta}_m}h-P_0h = o_p(1)$, we have that $P_{\hat{\eta}_m}h$ lies in the neighborhood of $P_0h$ in which $\Gamma$ is Lipschitz with probability tending to 1. Therefore, we have
\begin{align*}
    \|(P_{\hat{\eta}_m}-P_0)g_{\theta_0}\| &= \|\Gamma(P_{\hat{\eta}_m}h)-\Gamma(P_0h)\| \\
    &\le L \|P_{\hat{\eta}_m}h-P_0h\|
\end{align*}
where the inequality holds with probability tending to 1. If $\Gamma$ were instead Lipschitz globally, the above display holds with probability 1.
The triangle inequality then yields
\begin{align*}
    \|P_{\hat{\eta}_m}h-P_0h\|&\le \|P_{\hat{\eta}_m}h-P_mh\| + \|P_mh-P_0h\| \,.
\end{align*}
The first term is $O_p(m^{-1/2})$ by assumption, and the second term is $O_p(m^{-1/2})$ by the multivariate central limit theorem, using the assumption that $P_0\|h\|^2<\infty$. Therefore, we may conclude that $(P_{\hat{\eta}_m}-P_0)g_{\theta_0}=O_p(m^{-1/2})$.
\end{proof}

\section{Experiments (Supplement)}\label{sec:supp-experiments}

\subsection{Additional Experimental Results}\label{sec:supp-inf-experiments}

\begin{figure}[t!]
    \centering
    \includegraphics[width=\linewidth]{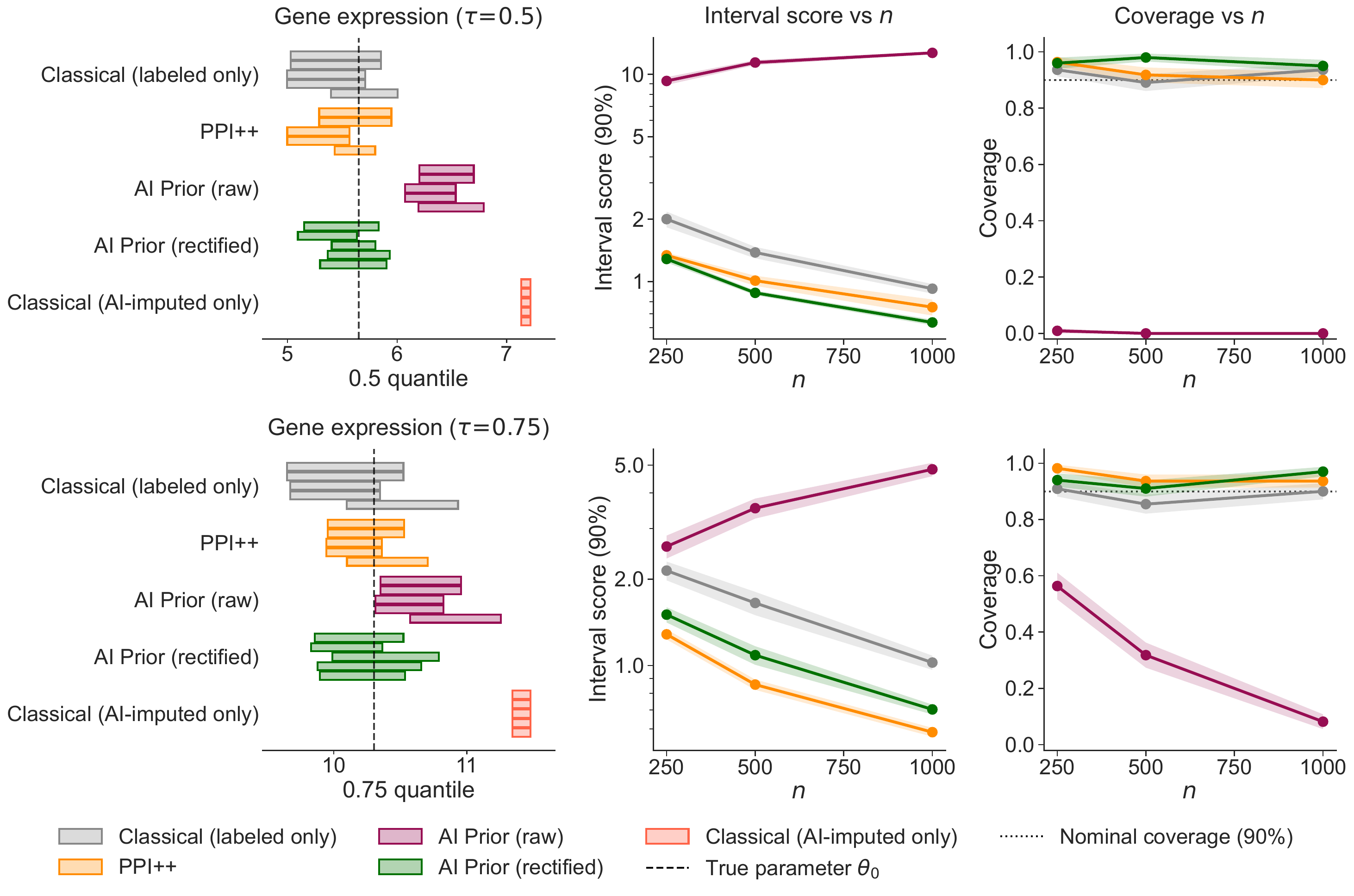}
    \caption{Results for gene expression on quantiles $\tau\in\{0.5,0.75\}$. The prior strength is fixed at $\gamma=1$. Error bars indicate one standard error.}
    \label{fig:supp-gene-quantile-ablation}
\end{figure}

\begin{figure}[t!]
    \centering
    \includegraphics[width=\linewidth]{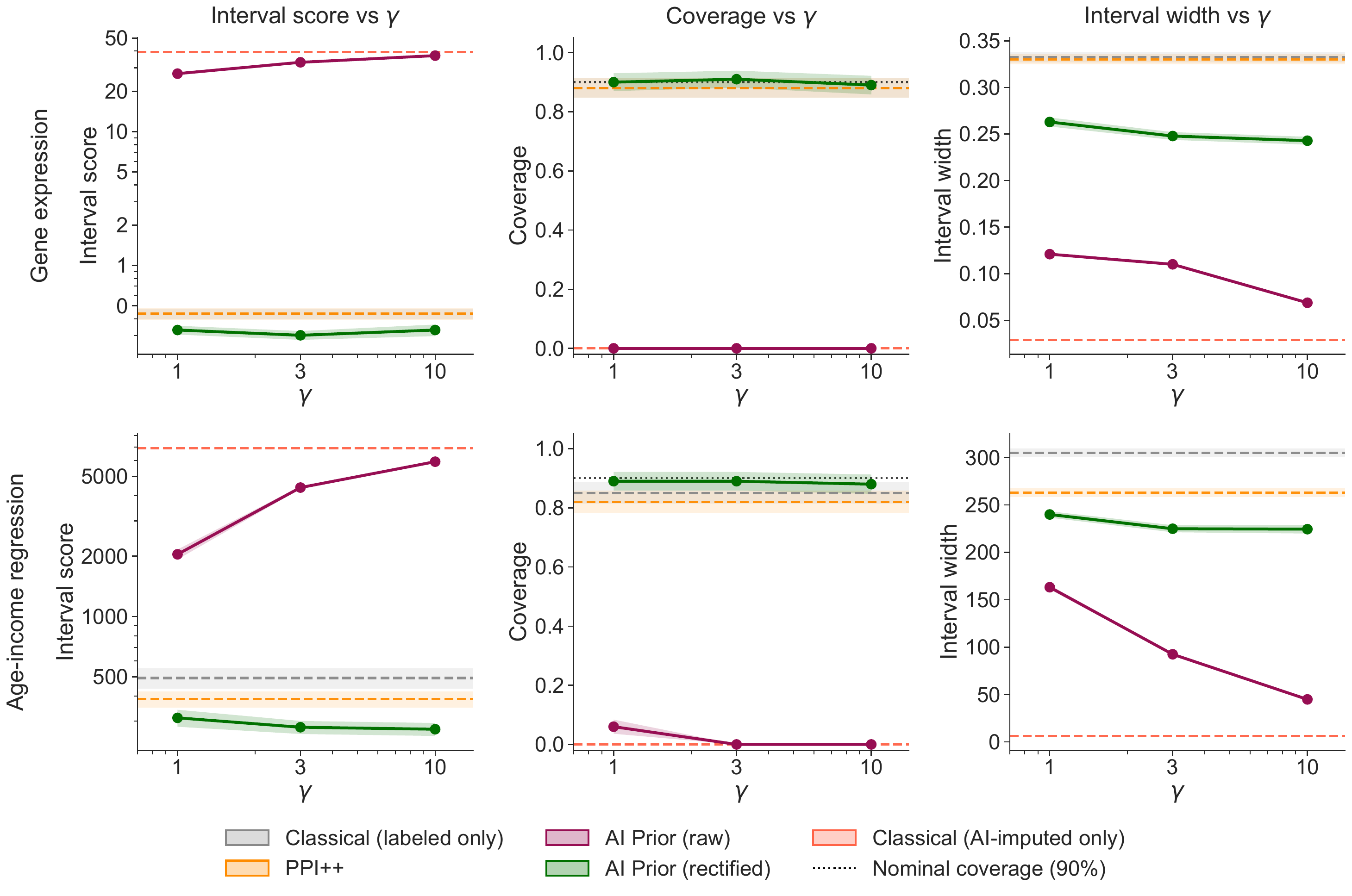}
    \caption{Results for gene expression and age-income regression experiments with varying prior strength $\gamma\in\{1,3,10\}$. The labeled sample size is fixed at $n=1000$. We show the average interval score, coverage, and interval width over 100 repetitions. Error bars indicate one standard error.}
    \label{fig:supp-gamma-ablation}
\end{figure}

\paragraph{Different prior strength values.} In~\cref{fig:supp-gamma-ablation}, we show the results of the experiments in~\cref{sec:exp-inference} with different choices of prior strength parameter $\gamma$. We observe that even as the prior strength increases massively (tenfold), we can retain coverage that is at the nominal level, or not too far below, while the raw AI prior absorbs too much bias from the base measure and causes the posterior to rapidly deteriorate in quality. 

\paragraph{Different quantiles for gene expression experiment.} \cref{fig:supp-gene-quantile-ablation} displays analogous results to the gene experiment in the main text, but for alternative quantile values $\tau\in\{0.5,0.75\}$. We find that rectification yields substantial improvements over the raw AI prior at this prior strength across all of the quantile values $\tau\in\{0.25,0.5,0.75\}$ investigated in~\citet{angelopoulos2023prediction}. The benefit is more exaggerated for the lower quantile values, where the AI base measure seems to incur more bias.

\begin{table}[h!]
    \centering
\resizebox{\textwidth}{!}{
\begin{tabular}{lll rrrr}
\toprule
Experiment & Rectifier & Strategy & Bias (SE) & IS (SE) & Width (SE) & Cov.\ (SE) \\
\midrule
  \multirow{10}{*}{Gene expression ($\tau\!=\!0.25$)} & None (raw) & n/a & 1.4 (0.0031) & 27 (0.049) & 0.12 (0.0015) & 0 (0) \\
  \cmidrule(lr){2-7}
   & \multirow{3}{*}{Moment matching} & Fixed & 0.0089 (0.011) & 0.60 (0.079) & 0.19 (0.0018) & 0.62 (0.049) \\
   &  & Split & 0.016 (0.011) & 0.43 (0.051) & 0.28 (0.0030) & 0.82 (0.038) \\
   &  & NPB & 0.012 (0.011) & 0.44 (0.054) & 0.27 (0.0042) & 0.85 (0.036) \\
  \cmidrule(lr){2-7}
   & \multirow{3}{*}{Isotonic} & Fixed & -0.069 (0.011) & 0.77 (0.088) & 0.22 (0.0093) & 0.59 (0.049) \\
   &  & Split & -0.065 (0.0097) & 0.48 (0.022) & 0.44 (0.0070) & 0.94 (0.024) \\
   &  & NPB & -0.068 (0.0100) & 0.51 (0.053) & 0.35 (0.0049) & 0.84 (0.037) \\
  \cmidrule(lr){2-7}
   & \multirow{3}{*}{Quantile map} & Fixed & 0.0076 (0.0078) & 0.32 (0.030) & 0.23 (0.0044) & 0.86 (0.035) \\
   &  & Split & 0.011 (0.0077) & 0.38 (0.0071) & 0.37 (0.0058) & 0.98 (0.014) \\
   &  & NPB & 0.0082 (0.0078) & 0.33 (0.025) & 0.26 (0.0048) & 0.90 (0.030) \\
\midrule
  \multirow{10}{*}{Age-income regression} & None (raw) & n/a & -180 (5.4) & 2000 (120) & 160 (2.2) & 0.060 (0.024) \\
  \cmidrule(lr){2-7}
   & \multirow{3}{*}{Moment matching} & Fixed & -6.4 (9.5) & 560 (76) & 170 (2.2) & 0.62 (0.049) \\
   &  & Split & -6.6 (9.5) & 420 (54) & 250 (3.1) & 0.80 (0.040) \\
   &  & NPB & -6.6 (9.5) & 480 (62) & 210 (3.3) & 0.74 (0.044) \\
  \cmidrule(lr){2-7}
   & \multirow{3}{*}{Isotonic} & Fixed & -46 (7.2) & 460 (59) & 170 (2.2) & 0.65 (0.048) \\
   &  & Split & -45 (7.2) & 360 (31) & 280 (3.7) & 0.90 (0.030) \\
   &  & NPB & -45 (7.2) & 410 (52) & 200 (2.5) & 0.74 (0.044) \\
  \cmidrule(lr){2-7}
   & \multirow{3}{*}{Quantile map} & Fixed & 2.8 (7.5) & 320 (36) & 210 (2.7) & 0.86 (0.035) \\
   &  & Split & 0.058 (7.4) & 340 (10) & 330 (3.9) & 0.98 (0.014) \\
   &  & NPB & 6.5 (7.4) & 310 (30) & 240 (3.4) & 0.89 (0.031) \\
\bottomrule
\end{tabular}
}

    \caption{Evaluation of rectifiers and calibration sample construction strategies on the gene expression and age-income regression examples over 100 repetitions. All combinations of rectifiers and strategies are generally able to meaningfully reduce the centering bias to a negligible fraction of that of the raw AI prior.}
    \label{tab:supp-rectifier-ablation}
\end{table}

\paragraph{Different rectification strategies.} \cref{tab:supp-rectifier-ablation} shows the results of the experiments in~\cref{sec:exp-inference} under different rectifiers and strategies for constructing the calibration sample used to fit the rectifier. Biases in the table are with respect to the posterior mean. The interval score (IS), width, and coverage are with respect to 90\% credible intervals. SE values indicate one standard error. The isotonic rectifier fits a monotone nondecreasing map on the calibration sample using the pool adjacent violators algorithm. Quantile map is defined in~\cref{sec:rectifier-taxonomy}. For moment matching, we fit a parametric rectifier (a scalar additive shift for gene expression, and an affine transformation for the OLS coefficient) by choosing its parameters to match the relevant empirical statistic on the calibration sample. Fixed, split, and NPB strategies refer to constructing the calibration sample by using the entire labeled sample, 50/50 sample splitting, and a nonparametric bootstrap sample from the labeled sample. All of the tested rectifier and strategy combinations substantially reduce the centering bias. In these applications, the quantile map rectifier tends to perform the best in terms of the interval score. Different calibration sample construction strategies do not meaningfully change the bias, but can affect the interval score under certain rectifiers. In particular, the fixed strategy inflates the interval score under the moment-matching and isotonic rectifiers. The choice of strategy does not make a large difference under the quantile map rectifier.

\subsection{Base measure for skin-disease classification}\label{sec:skin-disease-base-measure-supp}

\paragraph{Neural network structure and optimization.} The neural network structure consists of a single hidden layer of size 20, with ReLU activation and softmax output. The parameter of interest is the minimizer of the multivariate cross-entropy loss with respect to this architecture. For optimization, we use the Adam optimizer with 200 epochs on the full dataset (no mini-batching).

\paragraph{Covariate data and preprocessing.} The set of covariates is \texttt{age}, which is an integer number of years, \texttt{family\_history}, which is binary, and ten clinically observable features that describe the skin condition. These are: \texttt{erythema}, \texttt{scaling}, \texttt{definite\_borders}, \texttt{itching}, \texttt{koebner\_phenomenon}, \texttt{polygonal\_papules}, \texttt{follicular\_papules}, \texttt{oral\_mucosal\_involvement}, \texttt{knee\_and\_elbow\_involvement}, \texttt{scalp\_involvement}. Each of these are scored on a discrete ordinal scale from 0-3, where 0 means absent and 3 means severe or extensive. 

\paragraph{Base measure.} The base measure on covariates $P_{\mathrm{AI}}^X$ is the empirical distribution on the 235 unlabeled patients. 

The base measure on labels given covariates, $P_{\mathrm{AI}}^Y(\cdot\mid x)$ is given by prompting a large language model with patient information in the context. In particular, we query the \texttt{gpt-5.5} model via the OpenAI chat completions API, with \texttt{xhigh} reasoning effort. We use the following system prompt

\begin{verbatim}
  "You are a dermatology classification assistant. Given clinical
  features for one patient, estimate the probability of each of six
  erythemato-squamous diseases: psoriasis, seboreic dermatitis, lichen
  planus, pityriasis rosea, cronic dermatitis, pityriasis rubra pilaris.
  Most features are scored 0-3 (0 = absent, 3 = severe / extensive).
  family_history is binary (0 or 1). age is in years. Output strictly as
  a JSON object with these six keys; values are probabilities in [0, 1]
  and should sum to 1.0."
\end{verbatim}
and also supply a user message containing the covariate data as a comma separated string. For example:
\begin{verbatim}
    Patient features:\n erythema: 2, scaling: 2, […], age: 55
\end{verbatim}
We enforce a structured output with the following JSON schema:
\begin{verbatim}
  {
    "type": "object",
    "additionalProperties": false,
    "required": ["psoriasis", "seboreic dermatitis", "lichen planus",
                 "pityriasis rosea", "cronic dermatitis",
                 "pityriasis rubra pilaris"],
    "properties": {
      "<each label>": {"type": "number", "minimum": 0.0, "maximum": 1.0}
    }
  }
\end{verbatim}
This ensures that all six keys are present and in the range $[0,1]$. It does not enforce that probability vectors lie in the simplex, so we divide each given probability by their sum to ensure that they sum to one. We note that misspellings in the names of the diseases are inherited from the UCI dataset's class labels.

\subsection{Experimental compute resources}\label{sec:exp-comp-resources}

All of the experiments in the paper were conducted on an Apple MacBook Pro with M1 Pro processor and 16 GB memory. Six CPU workers were used, and the amount of time necessary to reproduce all of the experimental results is about four hours. This decomposes into: about two hours for all of the experimental results pertaining to the gene expression experiment (in the article and supplementary material), about one hour for all of the experimental results pertaining to the age-income regression experiment (in both the article and supplementary material), and about one hour for all results pertaining to the skin disease experiment.

\subsection{Licenses for existing assets}\label{sec:licenses}

Existing assets that we use in this article consist of datasets, and existing AI models for imputation of outcomes. We list these assets and provide citations and license information below.
\begin{itemize}
    \item The skin disease experiment uses the dermatology dataset from the UCI Machine Learning Repository~\citep{dermatology_33}. This dataset is licensed under the CC-BY 4.0 license. It can be obtained at this \href{https://archive.ics.uci.edu/dataset/33/dermatology}{web address}.
    \item Both the data and transformer model pertaining to the gene expression experiment originally come from~\citet{vaishnav2022evolution}. These assets are hosted under a CC-BY 4.0 license. We obtained the data and transformer imputations from a derived source, that is the \href{https://github.com/aangelopoulos/ppi_py}{\texttt{ppi-py} repository}, which is under an MIT license. 
    \item The age-income regression experiment uses Folktables and ACS PUMS data. The Folktables package is introduced by~\citet{ding2021retiring}, and under an MIT license. We obtained the data and AI-imputations from a derived source, that is the \href{https://github.com/aangelopoulos/ppi_py}{\texttt{ppi-py} repository}, which is under an MIT license. The ACS PUMS data are subject to the U.S. Census Bureau's terms of use~\citep{uscensus2019acspums}.
    
\end{itemize}



\renewcommand{\refname}{Additional References for the Appendix}
\putbib
\end{bibunit}

\end{document}